\documentclass{article}

     \PassOptionsToPackage{numbers, compress}{natbib}

     \usepackage[preprint]{neurips_2024}

\usepackage[utf8]{inputenc} 
\usepackage[T1]{fontenc}    
\usepackage{hyperref}       
\usepackage{url}            
\usepackage{booktabs}       
\usepackage{amsfonts}       
\usepackage{amsmath}
\usepackage{amsthm}
\usepackage{nicefrac}       
\usepackage{microtype}      
\usepackage{xcolor}         
\usepackage{graphicx}
\usepackage{float}
\usepackage{subcaption}

\newtheorem{thm}{Theorem}
\newtheorem{defn}{Definition}

\title{Infinite Width Models That Work: Why Feature Learning Doesn’t Matter as Much as You Think}

\author{%
  Luke Sernau \\
  Google DeepMind \\
  \texttt{sernau@google.com} \\
}

\begin{document}

\maketitle

\begin{abstract}
Common infinite-width architectures such as Neural Tangent Kernels (NTKs) have historically shown weak performance compared to finite models. This is usually attributed to the absence of feature learning. We show that this explanation is insufficient. Specifically, we show that infinite width NTKs obviate the need for feature learning. They can learn identical behavior by selecting relevant subfeatures from their (infinite) frozen feature vector. Furthermore, we show experimentally that NTKs under-perform traditional finite models even when feature learning is artificially disabled. Instead, we show that weak performance is at least partly due to the fact that existing constructions depend on weak optimizers like SGD. We provide a new infinite width limit based on ADAM-like learning dynamics and demonstrate empirically that the resulting models erase this performance gap.
\end{abstract}

\section{Introduction}

Neural Tangent Kernels (NTKs) \cite{ntk} represent one of the first attempts to capture the dynamics of training in the infinite width limit. Unfortunately, in practice they often under-perform finite-width models. The usual explanation given is that these infinite limits don’t allow feature learning \cite{no_feature_learning}.

By ``features,'' we mean the inputs to the last layer of a deep learning model. ``Feature learning'' then refers to the evolution of the rest of the model in order to learn useful feature representations. This language invites us to think of the entire first part of the model as merely setting up the last layer with the right representations to solve downstream tasks linearly.

In \cite{no_feature_learning}, Yang and Hu prove their ``Dynamical Dichotomy'', which states that under fairly general conditions, the only models that admit an infinite width kernel representation are those where the features do not evolve during training. The entire model except for the last layer is effectively frozen at initialization.

Since the bulk of practical infinite width schemes are kernel machines, this was taken as something of a death knell for efficient, expressive infinite models. How could infinite models hope to learn complex behavior if most of the model is untrained? Even a large feature vector will be mostly useless at initialization.

But ``infinite'' is very different from ``large,'' and sometimes our intuitions can lead us astray. At infinite width, feature vectors contain an infinite variety of structure, all of which is available to the final layer.

We’ll prove that merely by upweighting and downweighting subvectors, the final layer of an infinite width model has access to every possible behavior the previous layers could have learned. This can be viewed as an infinite extension to random kitchen sinks \cite{kitchen_sink} approach. Our contribution is showing that at the infinite limit, the effect of this freezing vanishes altogether.

We then offer a simpler explanation for the performance gap: common infinite width limits capture the dynamics of training under stochastic gradient descent (SGD), while the most successful finite models are trained using more sophisticated optimizers like ADAM \cite{adam}.

To address this issue, we provide a construction for infinite width limits under ADAM-like learning dynamics. This construction preserves the kernel representation of the limit, but eliminates the performance gap with finite models.

\section{Notation: Defining an Infinite Model}
We begin by writing an ordinary deep learning model in a way that makes it easier to think about both feature learning and infinite limits. For arbitrary input space \(X\) and output space \(Y = \mathbb{R}^E\), let \(\left(x_1, y_1\right), \left(x_2, y_2\right), \dots, \left(x_N, y_N\right) \in X \times Y\) be training examples. Let \(L:Y \times Y\rightarrow \mathbb{R}\) be a differentiable loss function.

Let \(H\) be the number of features. This is the quantity we will ultimately push to infinity.

Let \(f_{M\theta}: X\rightarrow \mathbb{R}^E\) be an arbitrary deep learning model, where \(M \in \mathbb{R}^{H \times E}\) denotes the weight matrix of the final layer and \(\theta \in \mathbb{R}^P\) denotes the remaining \(P\) parameters. We will use the function \(g_{\omega_h}: X\rightarrow \mathbb{R}\) to denote the submodel generating coordinate \(h\) of the feature vector. It is parameterized by some subset of the parameters \(\omega_h \subseteq \theta\) of size \(S\). (The subsets for different \(h\) need not be disjoint, but without loss of generality may be taken to be the same size.)

Altogether the model is given by
\begin{equation}\label{model}
    f_{M\theta}\left(x\right) = \frac{1}{\sqrt{H}}\sum^H_{h=1}g_{\omega_h}\left(x\right)M_h
\end{equation}
where without loss of generality we have scaled by \(1/\sqrt{H}\) in order to keep values finite when it comes time to scale \(H\).

We would like to train this model to minimize the loss over the training set, i.e. we seek
\begin{equation*}
    \min_{M,\theta}\sum^N_{i=1}L\left(f_{M\theta}\left(x_i\right), y_i\right).
\end{equation*}

Note that this is an extremely general formulation. \(g\) may be any model, including models which contain their own infinite width submodels. \(L\) is also arbitrary, and may include other nonlinearities, or even additional submodels of its own, which need not be frozen.

All we have really said is that our model contains a matrix \(M\) whose input dimension we would like to drive to infinity.

\subsection{A convenient notation for the infinite width limit}
Let \(\Omega\) be a discrete random variable uniformly distributed over \(\left\{\omega_1, \omega_2, \dots, \omega_H\right\}\in\mathbb{R}^S\). Then we can write \eqref{model} as an expectation,

\begin{equation}\label{expectation}
    f_{M\theta}\left(x\right) = \mathop{\mathbb{E}}_{\omega\sim\Omega}\left[g_{\omega}\left(x\right)\hat{M}_\omega\right].
\end{equation}
where \(\hat{M} = \sqrt{H}M\) and \(\hat{M}_\omega \in \mathbb{R}^E\) denotes the row of \(\hat{M}\) associated with \(\omega\) (for example, \(\hat{M}_{\omega_1}=\sqrt{H}M_1\)).

Written in this form, we could just as easily take \(\Omega\) to be a continuous random variable, conveniently expressing a model with infinite hidden dimension. In this view, \(\hat{M}\) is a possibly-nonlinear map from any \(\omega\) in the support of \(\Omega\) to a corresponding vector in \(\mathbb{R}^E\). 

As we'll see, this construction is equivalent to other notions of infinite model, but this notation has the convenient property of unifying over the finite and infinite cases depending on the choice of \(\Omega\). We may define \(\Omega\) over a set of any size, finite, infinite, discrete, or continuous.

\subsection{Infinite width models are the expected value of finite ones}

Another intriguing property of this notation is that it emphasizes a fundamental relationship between infinite and finite models. Let \(\omega_h\) be initialized by sampling from some continuous distribution \(\Omega\). Then from \eqref{model}, the expected value of the model (taken over the random initialization) will be

\begin{equation*}
    \mathop{\mathbb{E}}\left[\frac{1}{\sqrt{H}}\sum^H_{h=1}g_{\omega_h}\left(x\right)M_h\right]=\frac{1}{H}\sum^H_{h=1}\mathop{\mathbb{E}}_{\omega_h\sim\Omega}\left[g_{\omega_h}\left(x\right)\hat{M}_h\right]=\mathop{\mathbb{E}}_{\omega\sim\Omega}\left[g_{\omega}\left(x\right)\hat{M}_\omega\right]
\end{equation*}
This is exactly the expression in \eqref{expectation}. In other words, the expected value of a finite model is an infinite one.

This observation makes it possible to apply results about infinite models to finite ones taken in expectation. In the development that follows, we encourage the reader to consider the results from both of these perspectives.

\section{A Simpler Neural Tangent Kernel}
With this notation in hand, we begin our exploration with a simpler, more memory efficient development of Neural Tangent Kernels. The guiding intuition is to leverage Dynamical Dichotomy: Since NTKs have no feature learning, we don't give anything up by explicitly freezing the feature weights at initialization.

Concretely, Dynamical Dichotomy shows that NTK representations are only possible under parameterizations that leave each input's feature vector unchanged throughout training, effectively freezing the first part of the model at initialization. In the original development of NTKs \cite{ntk}, this effective freezing happens implicitly. We take this as an invitation to simply freeze \(\theta\) explicitly.

We’ll make a few other purely cosmetic assumptions to reduce bookkeeping, namely that we are training with batch size \(1\) and that \(M\) is zero initialized. Removing these assumptions is straightforward and left to the reader.

Now consider the effect of training \(M\). After \(t\) steps of training under SGD with step size \(\alpha\), the value of \(M\) is just the sum of the its updates under SGD.
\begin{align*}
    M_h &= -\alpha \sum^t_{i=1}\frac{\partial L\left(f_{M\theta}\left(x_i\right), y_i\right)}{\partial M_h}\\
    &= -\alpha \frac{1}{\sqrt{H}}\sum^t_{i=1}g_{\omega_h}\left(x_i\right)\frac{\partial L\left(f_{M\theta}\left(x_i\right), y_i\right)}{\partial f_{M\theta}\left(x_i\right)}
\end{align*}
where the second step follows by direct computation of the matrix derivative. It follows that \(\hat{M}_\omega\) is given by

\begin{equation}\label{derivative}
    \hat{M}_\omega= -\alpha \sum^t_{i=1}g_{\omega}\left(x_i\right)\frac{\partial L\left(f_{M\theta}\left(x_i\right), y_i\right)}{\partial f_{M\theta}\left(x_i\right)}.
\end{equation}

Plugging this into \eqref{expectation}, we get
\begin{equation*}
    f_{M\theta}\left(x\right) = -\alpha \mathop{\mathbb{E}}_{\omega\sim\Omega}\left[g_{\omega}\left(x\right) \sum^t_{i=1}g_{\omega}\left(x_i\right)\frac{\partial L\left(f_{M\theta}\left(x_i\right), y_i\right)}{\partial f_{M\theta}\left(x_i\right)}\right].
\end{equation*}
By changing the order of the summation, this can be written as
\begin{equation*}
    f_{M\theta}\left(x\right) = -\alpha\sum^t_{i=1} \mathop{\mathbb{E}}_{\omega\sim\Omega}\left[g_{\omega}\left(x\right) g_{\omega}\left(x_i\right)\right]\frac{\partial L\left(f_{M\theta}\left(x_i\right), y_i\right)}{\partial f_{M\theta}\left(x_i\right)}.
\end{equation*}
The expectation over \(g_{\omega}\left(x\right) g_{\omega}\left(x_i\right)\) is known as the Neural Tangent Kernel (NTK) of the model, often written as \(\Theta\left(x, x_i\right)=\mathop{\mathbb{E}}_{\omega\sim\Omega}\left[g_{\omega}\left(x\right) g_{\omega}\left(x_i\right)\right]\). Since \(\theta\) (and therefore \(\Omega\)) is frozen, \(\Theta\) is stationary throughout training.

Again, notice that because \(\Omega\) is permitted to be a discrete distribution, the NTK is defined for both finite and infinite models.

In the infinite-width limit, many common choices of \(g\) give rise to a \(\Theta\) that has a closed-form representation \cite{bietti2019inductive}\cite{daniely2016toward}, making it tractable to compute exactly. Altogether, our kernel machine is
\begin{equation}\label{kernel}
    f_{M\theta}\left(x\right) = -\alpha \sum^t_{i=1}\Theta\left(x, x_i\right)\frac{\partial L\left(f_{M\theta}\left(x_i\right), y_i\right)}{\partial f_{M\theta}\left(x_i\right)}.
\end{equation}

Note that unlike the standard development \cite{ntk} of NTKs, this kernel is compact relative to the size of the model, having in the finite case only \(H\) terms in definition of \(\Theta\), as opposed to \(HE+P\) in the original. It is also exactly stationary, even in the finite case. These properties are possible because we took \(\theta\) to be frozen. Dynamical Dichotomy shows that \(\theta\) is already \emph{effectively} frozen for infinite-width NTKs. All we have done is made this freezing explicit.

\section{Feature Learning for Neural Tangent Kernels}
Everything we've done so far assumes that \(\theta\) is frozen. Intuitively, we might expect that this kind of aggressive freezing would damage the model’s ability to learn.

Perhaps surprisingly, we show that in the infinite-width limit, freezing the model in this way has no effect on the expressiveness of the model. The reason for this is related to the lottery ticket hypothesis \cite{frankle2018lottery}. 

In large models, it is often the case that a subset of the parameters “get lucky” and end up close to a useful value. During training, the model can improve by simply pruning the unlucky ones. This effect becomes more pronounced as the model size increases, until in the continuous limit every feature representation exists with probability one. The final layer is free to simply select the subsets of the feature vector that are worthwhile, “training” the features without touching their weights. Unlike in the finite case, this filtering process need not reduce the cardinality of the features.

We show that this strategy is able to entirely replace gradient descent based feature learning, and that in fact gradient descent on the final matrix is enough to induce this strategy.

\begin{thm}\label{expressiveness}
For any infinite model as in \eqref{expectation}, suppose \(\Omega\) is totally supported and has a differentiable density. Then we may replace arbitrary updates to \(\Omega\) with appropriate updates to \(\hat{M}\) without changing model behavior or learning dynamics. It follows that we can freeze \(\Omega\) without affecting the final model quality.
\end{thm}
\begin{proof}
Since \(\Omega\) has a differentiable density \(\rho\), we may write \eqref{expectation} as
\begin{equation*}
f_{M\theta}\left(x\right) = \int_{\mathbb{R}^S}\rho\left(\omega\right)g_\omega\left(x\right)\hat{M}_\omega d\omega
\end{equation*}
Now consider applying some small update to \(\Omega\). Stated precisely, we allow its density \(\rho\) to flow along some vector field \(V\left(\omega\right)\), scaled by a learning rate \(\alpha\).  (In SGD, we would take \(V\) to be proportional to the negative gradient, \(V\left(\omega\right) = -\frac{\partial L}{\partial \omega}\), but our discussion will consider the general case.) How does this update affect our model?

Note that since \(V\) is arbitrary, we may break large updates into several small steps. This means that without loss of generality we make take \(\alpha\) to be infinitesimal.

The continuity equation \cite{pedlosky2013geophysical} then tells us that if the change in \(\Omega\) is \(V\) the change in the density \(\rho\) will be
\begin{align*}
    \frac{\partial \rho}{\partial t} &= -\nabla \cdot \left(\rho\left(\omega\right) V\right)\\
    &= -\left(\rho\left(\omega\right) \nabla \cdot V + V \cdot \nabla\rho\left(\omega\right)\right)\\
    &= -\rho\left(\omega\right) \left(\nabla \cdot V + V \cdot \nabla\log\rho\left(\omega\right)\right)\\
\end{align*}
via the product rule in the first step and the score function trick in the second.

After the update, the new model output \(f'_{M\theta}\) can be computed as
\begin{align*}
    f'_{M\theta}\left(x\right) &= \int_{\mathbb{R}^S}\left(\rho\left(\omega\right) + \alpha\frac{\partial \rho}{\partial t}\rho\left(\omega\right)\right)g_\omega\left(x\right)\hat{M}_\omega d\omega\\
    &= \int_{\mathbb{R}^S}\rho\left(\omega\right)\bigg(1 - \alpha \big(\nabla \cdot V + V\cdot \nabla\log\rho\left(\omega\right)\big)\bigg)g_\omega\left(x\right)\hat{M}_\omega d\omega\\
    &= \int_{\mathbb{R}^S}\rho\left(\omega\right)g_\omega\left(x\right)\hat{M}'_\omega d\omega\\
\end{align*}
where \(\hat{M}'_\omega = \hat{M}_\omega\bigg(1 - \alpha \big(\nabla \cdot V + V\cdot\nabla\log\rho\left(\omega\right)\big)\bigg)\). In other words, we can update \(\hat{M}\) rather than \(\rho\).
\end{proof}
This gives a procedure for converting updates to the whole model into equivalent updates on \(\hat{M}\) alone, demonstrating that we are free to leave the entire model frozen except for \(\hat{M}\) without compromising expressiveness.

While this scheme is useful for proving the existence of this expressiveness, a simple thought experiment shows that this expressiveness can also be accessed with ordinary gradient descent. As \(\alpha\) becomes small, the prescribed update to \(\hat{M}\) becomes first-order. So beneficial (i.e. loss reducing) updates and conventional gradient descent both use first-order updates in the direction of reduced loss. Such first-order improvements exist if and only if the gradient with respect to the loss is nonzero. So the set of converged (i.e. stationary) points for the beneficial updates as in Theorem \ref{expressiveness} is the same as those for gradient descent.

This means that if an update is beneficial, in the sense of producing lower loss, that benefit can also be obtained by conventional gradient descent on \(\hat{M}\).

\section{Infinite Width Learning with ADAM}
If feature learning isn’t the problem, what's causing the performance gap between NTK models and conventional models? The answer is simpler than it might seem. NTK optimization has historically been designed to mimic the update performed in SGD. The derivation assumes that weight updates are proportional to the negative gradient.

Modern machine learning is almost exclusively done using optimizers that support, at a minimum, momentum. These methods are more effective in practice than naive SGD. We give an NTK construction that models ADAM-style gradient updates. As we’ll see, this closes the performance gap.

\begin{defn}[ADAM]
Recall that for a given gradient \(J_t\), the ADAM update \(u_t\) for training step \(t\) can be calculated as
\begin{align*}
    m_0 &= 0\\
    v_0 &= 0\\
    m_t&=\beta_1m_{t-1}+\left(1-\beta_1\right)J_t\\
    v_t&=\beta_2v_{t-1}+\left(1-\beta_2\right)J_t^2\\
    u_t&=-\alpha \frac{m_t}{1-\beta_1^t}\cdot\left(\frac{v_t}{1-\beta_2^t}\right)^{-1/2}\\
\end{align*}
where \(\beta_1\) and \(\beta_2\) are decay hyperparameters, and all operations are applied coordinatewise.
\end{defn}

Unfortunately, we can’t construct a NTK infinite limit based on this definition as-is. \(J_t\), \(m_t\), and \(v_t\) all grow with \(H\), requiring infinite memory in the infinite limit. In the case of SGD, we were able to find a rewriting that meant we never had to write down \(J_t\) explicitly.

Getting a similar rewriting to work for ADAM requires a minor modification. We introduce ADAM*, a variant that is built to preserve the useful learning dynamics of ADAM while accommodating infinite limits.

\begin{defn}[ADAM*]\label{adam*}
For a given gradient \(J_t\), the ADAM* update \(u_t\) is the same as ADAM, except that \(v_t\) is replaced with its expectation over \(\Omega\). (For a finite model, this corresponds to averaging over any axis of dimension \(H\).)
\begin{align*}
    m_0 &= 0\\
    v_0 &= 0\\
    m_t&=\beta_1m_{t-1}+\left(1-\beta_1\right)J_t\\
    v_t&=\beta_2v_{t-1}+\left(1-\beta_2\right)\mathop{\mathbb{E}}_{\omega \sim \Omega}\left[J_t^2\right]\\
    u_t&=-\alpha \frac{m_t}{1-\beta_1^t}\cdot\left(\frac{v_t}{1-\beta_2^t}\right)^{-1/2}
\end{align*}
where the coordinate-wise multiplication in the last step is taken to be broadcast (as it is now between a \(H \times E\) matrix and a \(E\)-dimensional vector).
\end{defn}

This change is enough to allow us to build an NTK that follows ADAM* dynamics.

\begin{thm}\label{adam}
Let \(f_{M\theta}\) be defined as in \eqref{expectation}. Let \(M\) be zero-initialized and trained via ADAM* for \(t\) steps. Let \(\Theta\left(x, x_i\right)=\mathop{\mathbb{E}}_{\omega\sim\Omega}\left[g_{\omega}\left(x\right) g_{\omega}\left(x_i\right)\right]\) be the NTK. Then

\begin{equation}\label{ADAMkernel}
    f_{M\theta}\left(x\right) = -\alpha \sum^t_{i=1}\Theta\left(x, x_i\right)\frac{\partial L\left(f_{M\theta}\left(x_i\right), y_i\right)}{\partial f_{M\theta}\left(x_i\right)}c_i.
\end{equation}
where \(c_i = \sum^t_{j=i}\beta_1^{j-i}\frac{1-\beta_1}{1-\beta_1^j}\left(\frac{\hat{v}_j}{1-\beta_2^j}\right)^{-1/2}\)
and \(\hat{v}_j\) is defined recursively as
\begin{equation*}
\hat{v}_0 = 0, \hspace{3em} \hat{v}_j = \beta_2 \hat{v}_{j-1}+\left(1-\beta_2\right)\Theta(x_j, x_j)\left(\frac{\partial L\left(f_{M\theta}\left(x_j\right), y_i\right)}{\partial f_{M\theta}\left(x_j\right)}\right)^2,
\end{equation*}
with all operations performed coordinate-wise.

\end{thm}
Note that \(c_i\) is an \(E\)-dimensional vector whose scale is independent of \(H\). Provided we have an \(O(1)\) way to compute \(\Theta\), this entire computation may be performed in \(O(t)\) time, even in the infinite limit. Just as for traditional NTKs, the computation is possible in both the infinite and the finite case.

But unlike traditional NTKs, this NTK captures the momentum and adaptive step size that make ADAM so effective. As we'll show, this leads to tangible improvements.

\section{Experimental Results}
While this work is intended primarily as a theoretical exploration, we provide some cursory experimental validation. We show experimentally that NTK models exhibit a performance gap even when compared against models with no feature learning, demonstrating that feature learning alone cannot explain the poor performance. We then show that using the ADAM* NTK from Definition \ref{adam*} closes this gap.

Our model is a six layer, decoder-only transformer model \cite{vaswani2017attention} with 16 attention heads and an embed dimension of 512, trained on C4 for 40k steps. We replace the MLPs in the model with one of a few possibilities.
\begin{itemize}
  \item The original MLP, unfrozen.
  \item An MLP of the same size with no feature learning (i.e. frozen first matrix).
  \item A (finite) NTK of the same width, using the standard construction.
  \item An NTK of the same width using the ADAM* construction (ours).
\end{itemize}

\begin{figure}[H]
  \centering
  \includegraphics[width=.65\linewidth]{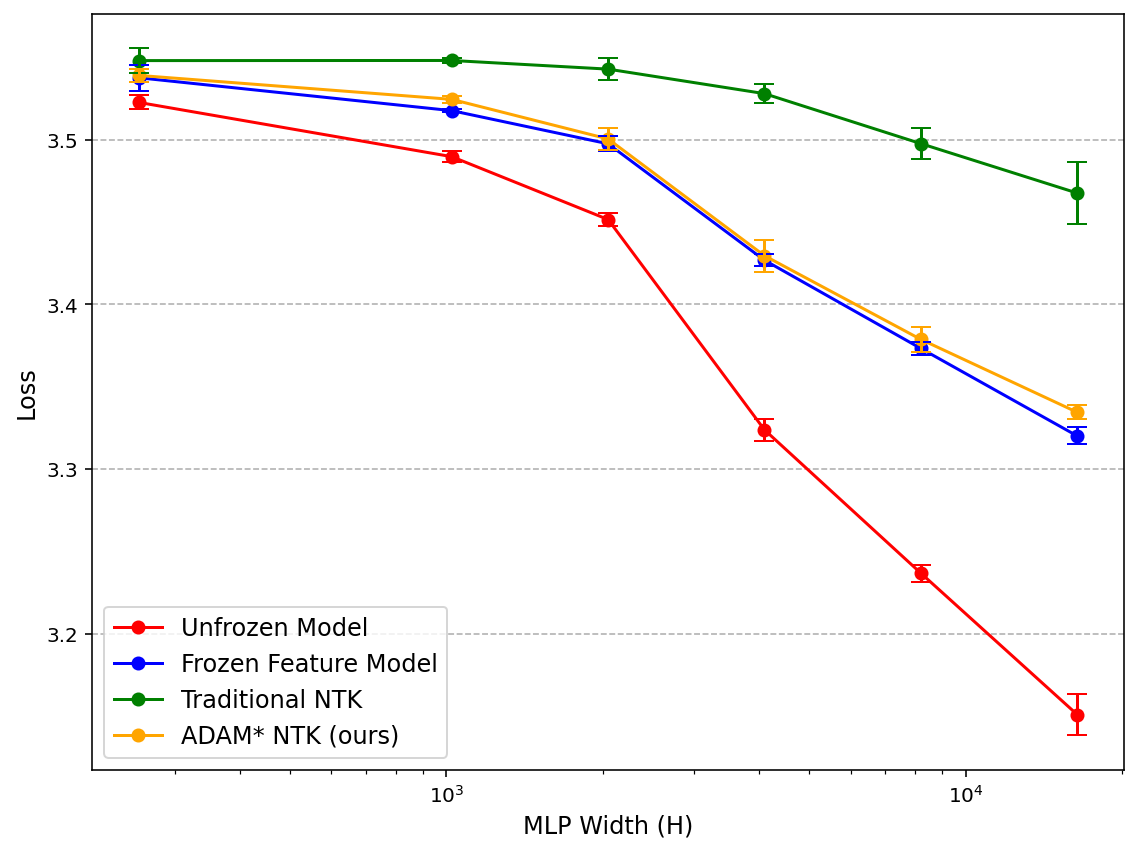}
  \caption{Final loss at various MLP widths for different models and NTKs.}
\end{figure}

Notice that the traditional NTK under-performs an ordinary MLP even when the MLP has no feature learning. This demonstrates that feature learning cannot be the root cause of the performance drop.

On the other hand, our ADAM* NTK has access to momentum and adaptive learning rate, and is able to closely track the performance of the frozen feature model across a wide variety of scales. In the infinite-width limit, our theoretical analysis suggests that the performance of both ADAM variants should ultimately converge with the unfrozen model, but we leave experimental exploration of this phenomenon to future work.

\bibliography{InfiniteWidthModelsThatWork}


\appendix

\section{Appendix}

\begin{proof}[Proof of Theorem \ref{adam}]
We'll first prove that \(\hat{v}\) as defined in Theorem \ref{adam} is related to \(v\) in Definition \ref{adam*}. Specifically, we'll show that \(\hat{v} = Hv\).

Since \(J\) is the gradient of the loss with respect to \(M\), the value of row \(h\) at time \(t\) will be
\begin{equation*}
    J_{t\omega_h}=\frac{\partial L\left(f_M\theta\left(x_t\right),y_t\right)}{\partial M_h}=\frac{1}{\sqrt{H}}g_{\omega_h}\left(x_t\right)\frac{\partial L\left(f_M\theta\left(x_t\right),y_t\right)}{\partial f_{M\theta}\left(x_t\right)}.
\end{equation*}
It follows that
\begin{align*}
\mathop{\mathbb{E}}_{\omega \sim \Omega}\left[J_{t\omega}^2\right] &= \frac{1}{H}\mathop{\mathbb{E}}_{\omega \sim \Omega}\left[g_\omega\left(x_t\right)^2\left(\frac{\partial L\left(f_M\theta\left(x_t\right),y_t\right)}{\partial f_{M\theta}\left(x_t\right)}\right)^2\right]\\
&= \frac{1}{H}\mathop{\mathbb{E}}_{\omega \sim \Omega}\left[g_\omega\left(x_t\right)^2\right]\left(\frac{\partial L\left(f_M\theta\left(x_t\right),y_t\right)}{\partial f_{M\theta}\left(x_t\right)}\right)^2\\
&=\frac{1}{H}\Theta\left(x_t, x_t\right)\left(\frac{\partial L\left(f_M\theta\left(x_t\right),y_t\right)}{\partial f_{M\theta}\left(x_t\right)}\right)^2.
\end{align*}
The result then follows by induction on the definition.

We turn our attention to the momentum \(m\). Adding up all terms from \(0\) to \(j\),
\begin{align*}
m_j &= \left(1-\beta_1\right)\sum_{i=1}^j\beta_1^{j-i}J_i.
\end{align*}
\(M\) is zero-initialized, so it's equal to the sum of its updates, \(M = \sum_{j=1}^t u_j\). \(\hat{M}\) is then
\begin{align*}
\hat{M} &= \sqrt{H}M\\
&= \sqrt{H}\sum_{j=1}^t u_j\\
&=-\alpha \sqrt{H}\sum_{j=1}^t \frac{m_j}{1-\beta_1^j}\cdot\left(\frac{v_j}{1-\beta_2^j}\right)^{-1/2}\\
&=-\alpha \sum_{j=1}^t \frac{m_j}{1-\beta_1^j}\cdot\left(\frac{\hat{v}_j}{1-\beta_2^j}\right)^{-1/2}\\
&=-\alpha \sum_{j=1}^t \left(\left(1-\beta_1\right)\sum_{i=1}^j\beta_1^{j-i}J_i\right)\frac{1}{1-\beta_1^j}\cdot\left(\frac{\hat{v}_j}{1-\beta_2^j}\right)^{-1/2}\\
&=-\alpha \sum_{i=1}^tJ_i \sum_{j=i}^t \beta_1^{j-i}\frac{1-\beta_1}{1-\beta_1^j}\cdot\left(\frac{\hat{v}_j}{1-\beta_2^j}\right)^{-1/2}\\
&=-\alpha \sum_{i=1}^tJ_i c_i\\
\end{align*}
where in the second to last step we exchange the order of summation.

From here, the proof proceeds much as it did for the ordinary NTK. From \eqref{expectation} we have
\begin{align*}
    f_{M\theta}\left(x\right) &= \mathop{\mathbb{E}}_{\omega\sim\Omega}\left[g_{\omega}\left(x\right)\hat{M}_\omega\right]\\
    &= -\alpha \mathop{\mathbb{E}}_{\omega\sim\Omega}\left[g_{\omega}\left(x\right)\sum_{i=1}^tJ_{i\omega} c_i\right]\\
    &= -\alpha \sum_{i=1}^t\mathop{\mathbb{E}}_{\omega\sim\Omega}\left[g_{\omega}\left(x\right)J_{i\omega} c_i\right]\\
    &= -\alpha \sum_{i=1}^t\mathop{\mathbb{E}}_{\omega\sim\Omega}\left[g_{\omega}\left(x\right)g_\omega\left(x_i\right)\frac{\partial L\left(f_M\theta\left(x_i\right),y_i\right)}{\partial f_{M\theta}\left(x_i\right)} c_i\right]\\
    &= -\alpha \sum_{i=1}^t\mathop{\mathbb{E}}_{\omega\sim\Omega}\left[g_{\omega}\left(x\right)g_\omega\left(x_i\right)\right]\frac{\partial L\left(f_M\theta\left(x_i\right),y_i\right)}{\partial f_{M\theta}\left(x_i\right)} c_i\\
    &= -\alpha \sum_{i=1}^t\Theta\left(x, x_i\right)\frac{\partial L\left(f_M\theta\left(x_i\right),y_i\right)}{\partial f_{M\theta}\left(x_i\right)} c_i,\\
\end{align*}
as desired. Note that this is essentially the same development as before, except for the introduction of \(c_i\).

\end{proof}

\end{document}